\documentclass{article}

\usepackage{graphicx}
\usepackage{booktabs}

\usepackage{amsmath}
\usepackage{amsthm}
\usepackage{amsfonts}
\usepackage{amssymb}

\usepackage{algorithm}
\usepackage{algorithmic}

\usepackage{comment}
\usepackage{caption}
\usepackage{subcaption}

\usepackage{mathtools}

\newtheorem{thm}{Theorem}[section]

\newtheorem{lem}[thm]{Lemma}

\newcommand{\spaceo}{\hspace{2 mm}}
\newcommand{\setsep}{ \spaceo | \spaceo}

\newcommand{\Prob}[1]{\mathbb{P}\left( #1 \right)}
\newcommand{\Probu}[2]{\mathbb{P}_{#1}\left( #2 \right)}
\newcommand{\argmax}{\operatornamewithlimits{argmax}}

\newcommand{\Abs}[1]{\left| #1 \right|}
\newcommand{\Set}[1]{\left\{ #1 \right\}}
\newcommand{\Brack}[1]{\left( #1 \right)}

\newcommand{\norm}[1]{\left\|#1\right\|}

\newcommand{\cdt}{ \mathcal{C}^{DT}}
\newcommand{\ctw}{ \mathcal{C}^{TW}}
\newcommand{\ct}{ \mathcal{C}^{T}}
\newcommand{\cd}{ \mathcal{C}^{D}}

\newcommand{\cS}{S}
\newcommand{\cW}{W}
\newcommand{\cP}{\mathcal{P}}

\newcommand{\BD}{\mathcal{D}}

\begin{document}

\title{Interdependent Gibbs Samplers}

\author{
  Mark Kozdoba \\
  \texttt{markk@technion.ac.il}
 \and
 Shie Mannor\\
  \texttt{shie@ee.technion.ac.il}
}

\date{}

\maketitle

\begin{abstract}
Gibbs sampling, as a model learning method, is known to produce the most accurate results available in a 
variety of domains, and is a de facto standard in these domains. Yet, it is also well known that 
Gibbs random walks usually have bottlenecks, sometimes termed ``local maxima'', and thus samplers often return 
suboptimal 
solutions. In this paper we introduce a variation of the Gibbs sampler which yields high likelihood 
solutions significantly more often than the regular Gibbs sampler.  
 Specifically, we show that combining multiple 
samplers, with certain dependence (coupling) between them, results in higher likelihood solutions.  This side-steps the well known issue of identifiability, which has been the  
obstacle to combining samplers in previous work. 
We evaluate the approach on a Latent Dirichlet 
Allocation model, and also on HMM's, where precise computation of likelihoods and comparisons to the standard 
EM algorithm are possible.
\end{abstract}

\section{Introduction}
\label{sec:intro}

Gibbs sampling is a standard model learning method in Bayesian Statistics, and in particular in the field of Graphical Models, 
\cite{gelman_bayesian}. In the Machine Learning community, it is commonly applied in situations where non sample based 
algorithms, such as gradient descent and EM are not feasible. Perhaps the most prominent application example is the Latent 
Dirichlet Allocation (LDA) topic model \cite{Blei02latentdirichlet}, a model that is routinely used in text analysis, 
recommender systems, and a variety of other domains.  For LDA, Gibbs sampling is considered to be the standard and the most 
accurate approach (see, for instance, \cite{smola_parallel}, \cite{papanikolaou}).

\subsection{Local Maxima}
The standard asymptotic theory of random walks guarantees that 
Gibbs, and more generally, MCMC samplers, eventually produce a sample from 
the target distribution (see \cite{gelman_bayesian}). The number of steps 
of the sampler that is required to approximate well a sample from the 
target distribution is known as the \textit{mixing time} of the sampler. However, 
it is well known to practitioners that the true mixing times of Gibbs 
samplers in realistic situations are unpractically high. The typical 
observed behavior of the random walk is not that one obtains a fresh 
sample from the target distribution every given number of steps, but 
rather that the walk converges to a certain configuration of the 
parameters, and stays in the vicinity of these parameters for extremely 
long times, which would not have had happened if mixing times where 
short. This is illustrated in Figure \ref{fig:c1_data_likelihoods}, and a 
related discussion is given in Section \ref{sec:hmm_experiments}.

In view of this behavior, it is natural to ask how one can improve the 
quality of the solutions obtained as the local maxima. The simplest 
classical receipt is to run the sampler multiple times independently, 
and to choose the best solution. Note that it is usually not clear that 
running a sampler a few or even a few hundred times will significantly 
improve the quality of a solution. Moreover, in some situations, such as 
LDA, even the task of computing the quality of the solution (the 
loglikelihood in this case) is hard and can not be performed reliably.

Another approach would be to attempt to combine multiple solutions, to 
produce a single improved solution. However, while this general idea 
was discussed already in \cite{cgibbs}, at present there is no known 
way to carry this out, due to an issue known as the identifiability 
problem. In the next section we introduce the approach taken in this 
paper, and its relation to identifiability.

\subsection{Mulitpaths and Identifiability}
The approach of this paper, which we call the multipath sampler,  is to use multiple, \textit{interdependent} Gibbs samplers. 

Consider a Bayesian setting
where we a have generative model for the observed data
$w =(w_1,\ldots, w_N)$ and a set of latent variables $p=(p_1,\ldots,p_N)
$. The model depends on a set of parameters $\phi$, and this set 
is endowed with a prior, $\Prob{\phi}$. We call the set of latent 
variables a \textit{path}.  For LDA, $\phi$ corresponds to topics, with a 
Dirichlet prior, and for every token $w_i$, $i \leq N$, $p_i$ is the 
topic assigned to that token. 
In a standard (partially collapsed) Gibbs sampler one first samples the 
parameters from $\Prob{\phi | p, w}$, and then 
samples topic assignments $p_i$,  from $\Prob{p_i | p_{-i}, w, \phi}$, 
for each $i \leq N$. Here $p_{-i}$ are the values of all the assignments in the path except $p_i$.

In the multipath sampler, instead of a single set of latent variables
$p$ we maintain 
$m$ sets, $p^1, \ldots, p^m$.  We define an 
appropriate joint distribution on them as described in Section 
\ref{sec:multiple_paths}. Under this joint distribution, the marginal 
distribution $p^j$ for each $j \leq m$ given the data $w$, coincides with 
the distribution of latent variables $p$ in the regular Gibbs sampler. 

Next, to sample the latent variables $p^j$, we first sample
$\Prob{\phi | p^1, \ldots, p^m, w}$, and then use the usual Gibbs sampler 
independently for each $p^j$, given $\phi$. 
The dependence between the paths $p^j$ is expressed in sampling from $
\Prob{\phi | p^1, \ldots, p^m, w}$, 
where $\phi$ depends on \textit{all} the $p^j$ together. This step can be 
viewed as model averaging.  The details are given in Section 
\ref{sec:multiple_paths}. In particular, we will show that the target 
distribution of the multipath sampler emphasizes the solutions that 
have high likelihood for the regular sampler.

As noted earlier, the idea of model averaging in the context Gibbs 
sampling has been discussed in the literature at least 
since the paper \cite{cgibbs}. In particular, it is well known that it is 
generally impossible to average the results of independent 
samplers, due to the problem of \textit{identifiability} (see, for 
instance, \cite{cgibbs}, \cite{cgibbs_book}, 
\cite{gelman_bayesian}), which refers to the issue that the model 
parameters are only defined by 
the model up to a permutation. For instance, in the case of LDA, 
suppose $\phi = (\phi_1,\ldots,\phi_T)$ and 
$\phi' = (\phi'_1,\ldots,\phi'_T)$ are two sets of topics estimated by two 
independent Gibbs samplers. 
One then might attempt to combine $\phi_1$ and $\phi'_1$, for instance 
by defining $\phi''_1 = \frac{\phi_1 + \phi'_1}{2}$, to obtain
a better estimate of the topic. However, even if $\phi$ and $\phi'$ 
represent nearly the same set of topics, $\phi_1$ in
$\phi$ does not necessarily correspond to $\phi'_1$, but may rather 
correspond to some other $\phi'_j$ for some $j \neq 1$.  Thus, to perform
direct averaging, one would have to find a proper permutation expressing 
the correspondence, before averaging. 
Moreover, the situation is further complicated by the fact that $\phi$ 
and $\phi'$ do not necessarily represent 
similar sets of topics, and thus for many topics in $\phi$, there simply 
will be no ``similar'' topic in $\phi'$ to average with.

With the multipath approach, we show for the first time how the 
identifiability problem can be avoided. Indeed, instead of running the 
samplers completely independently and attempting to only combine the 
final results, we allow the samplers to share the 
parameters \textit{during} the optimization. This forces the samplers to 
find a common version of each topic, while still maintaining different 
and somewhat independent empirical estimates of it.

\subsection{Experiments}
\label{sec:intro_experiments}

We demonstrate the methods introduced in this paper on HMM and LDA 
Gibbs samplers, and in both cases the multipath sampler improves on the 
results of a regular sampler. We now motivate and briefly describe the 
experiments. 

As mentioned earlier, in a typical problem where a Gibbs sampler is used,
given the model parameters $\phi$ found by a sampler, there is no precise 
way to evaluate the likelihood of the data $w$ with respect to $\phi$. 
See \cite{wallach_evaluation}, where an equivalent notion of perplexity is
discussed. As a result, it is not usually easy to compare two different 
solutions, $\phi$ and $\phi'$, possibly found by two different methods, 
and perplexity, if used at all, is never used as a sole method of 
evaluation.

Therefore, to evaluate the performance of the multipath sampler we first 
compare it to the regular Gibbs sampler on a model where explicit
computation of data likelihoods is possible. Specifically, we use 
Hidden Markov Models (HMMs), with synthetic data, where the parameters 
$\phi$ to be estimated are the emission and transition probabilities. We 
compare the solutions form the Gibbs samplers to those found by the 
EM-type Baum-Welch algorithm, and to the ground truth.

Next, we compare the Gibbs and the multipath Gibbs sampler LDA model on
synthetic data and measure the closeness of the recovered topics 
to the ground truth. 

In addition, on the well known State of The Union corpus, \cite{TOT}, we 
compare the \textit{time concentration} of the topics as an 
additional quality indicator. 
Specifically, this dataset consists of 
speeches given between the years 1790 to 2013, and each speech is 
further divided into multiple documents. When such time information is 
available, it is natural to ask whether certain topics appear in a 
specific time period, or their occurrence is spread relatively uniformly 
over time.  We call a topic time-concentrated if most of its occurrences 
are contained in a short time period. Consider Figure 
\ref{fig:welfare_topic_compare} for an example of a relatively well 
time-concentrated topic. 

Next, note that the LDA model itself has no access to timestamps, and 
that if token assignments were given at random, then the topics would 
be approximately uniformly spread over time. Therefore, time 
concentration of the topics may be observed only if the model fits the 
data well, and better time concentration implies a better fit to the 
data. In Section \ref{sec:lda_experiments} we show that topics and 
topic assignments found by the multipath sampler have significantly 
better time-concentration properties compared to topics from the regular 
sampler. 

We note that time concentration for topics is a well studied subject. For 
instance, the well known Topics over Time (ToT) model, \cite{TOT}, and 
its numerous extensions, incorporate the timestamp data into the model 
via appropriate probabilistic constraints. These additional 
constraints may be viewed as an expression of a prior belief about the 
data, in particular that a topic's usage pattern has certain temporal 
properties. Using such models, one indeed obtains topics that are better 
concentrated in time compared to standard LDA. However, note that 
by extending the model, one effectively forces a bias towards learning
concentrated topics. In contrast, here we show that one can obtain 
improved concentration by an improved learning algorithm of LDA itself, 
\textit{without} additional assumptions. This demonstrates that 
time-concentration of the topics is a signal contained in the data, 
rather than an artifact of the extended model.

The rest of the paper is organized as follows: In Section 
\ref{sec:literature} we review the literature. In Section 
\ref{sec:multiple_paths} we formally define the mulipath sampler and 
study its relation to the regular sampler. Section 
\ref{sec:experiments} contains the experiments, and concluding remarks 
and future work are discussed in Section \ref{sec:conclusion}.

\section{Literature}
\label{sec:literature}
A few days after the first posting of this paper, we have been informed 
that a similar approach has already been considered previously. While the 
experiments and the point of view taken in the discussion in this paper 
are different, the core idea, Algorithm \ref{alg:m_path_gibbs}, is 
identical to the algorithm introduced in \cite{same_alg}.

As detailed below, there is a large body of work on performance improvements for Gibbs samplers, and in particular for the Gibbs samplers 
of the LDA model. 

The uncollapsed Gibbs sampler for LDA was introduced in \cite{uncollapsed_gibbs} where it was applied to genetic data. 
The collapsed Gibbs sampler was introduced in \cite{cgibbs}. 

Parallelization was studied, for instance, in \cite{Newman_parallel},  
\cite{smola_parallel}. In \cite{parallel_mcmc_neis} 
parallelization methods with asymptotical exactness guarantees were 
given. Note that while 
parallelizing the fully collapsed Gibbs sampler presents 
several significant issues which are addressed in the above work, for LDA 
there is a partially collapsed sampler, studied in
\cite{sparse_patially_collapsed_mcmc11}, which is straightforward to 
parallelize. As observed in \cite{sparse_patially_collapsed_mcmc15}, the 
partially collapsed samplers may have a performance, in terms of quality 
of solutions, comparable to the performance of the fully collapsed 
sampler.

Methods that exploit sparseness of distributions to speed up the 
iteration of the sampler were studied, among other work, in 
\cite{Yao_sarsity}, \cite{fast_gibbs1}, and further optimized in 
\cite{warp_lda}.  These methods perform the same random walk as the 
collapsed sampler, but the sampling step itself is implemented more 
efficiently.

In a recent work \cite{papanikolaou}, parameter inference by averaging several solutions was discussed for the case of LDA. 
It is important to note here that this refers to a very restricted form of averaging. Indeed, in that work, one first 
obtains a solution from a collapsed Gibbs sampler, and a postprocessing step  considers perturbations in this solution 
\textit{at a single coordinate}, to obtain some variability in the assignment. Thus all the perturbed solutions differ from 
the original solution at a single coordinate at most. It is then argued that such perturbations can be averaged, since 
changing only one assignment coordinate in a large corpus is unlikely to cause the identifiability issues associated with 
averaging. Clearly, however, due to the same reasons, such perturbations can not change the particular local maxima of the 
original solution and can only provide small variations around it. 
Thus, while this approach is useful as a local denoising 
method, especially in computing topic assignment distributions (rather than topic themselves), and is suggested as such in \cite{papanikolaou}, the overall quality of the solution is the same as the one returned by the original sampler. As discussed in the Introduction, our approach 
completely avoids the identifiability issues, and allows to combine information from significantly differing topic
assignments. 

Finally, we note that all of the above mentioned methods can be combined 
with the methods introduced in this paper. In particular both multipath 
samplers may be parallelized and may incorporate sparsity in exactly the 
same way this is done for the regular sampler.

\section{Multiple Paths}
\label{sec:multiple_paths}

We will assume that the data is generated in a standard Bayesian setting: 
A set of model parameters is denoted by $\Phi$,  endowed with a prior 
distribution $\Prob{\phi}$. Given $\phi$, the model specifies a 
distribution of latent variables $p = \Brack{p_i}_{i=1}^N$, $\Prob{p | 
\phi}$. Each $p_i$ takes values in the set $\Set{1,\ldots,\cS}$. 
By analogy to HMMs, we refer to the latent variables $p$ as a 
\textit{path}, and denote the set of all possible values of a path 
sequence by $\cP = \cP_N = \Set{1,\ldots,\cS}^N$. 

Finally, given the parameters and the path, the model specifies the 
distribution of the observed variables $\Prob{w | p, \phi}$, 
where $w = \Brack{w_i}_{i=1}^N$. We also refer to $w$ as the data 
variables. Each $w_i$ takes values in the alphabet  $\Set{1,\ldots,\cW}$.

In the case of HMMs, the parameters $\phi$ are be the transition and 
emission distributions. Specifically, we consider HMMs with $\cS$ states, 
and with emissions given by discrete distributions on the alphabet 
$\Set{1,\ldots,\cW}$. Then 
$\phi = (\beta', \beta_1, \ldots, \beta_{\cS}, e_1, \ldots, e_{\cS})$, 
where $\beta'$ is the initial distribution on the Markov chain, 
for each $j \leq \cS$, $\beta_j$ is the transition distribution given the 
chain is at state $j$, and $e_j$ are emission distributions for state 
$j$. The prior is given by $\beta' \sim Dir_{\cS}(1)$, $\beta_j \sim 
Dir_{\cS}(1)$, and $e_j \sim Dir_{\cW}(1)$. The path $p$ is a state 
sequence sampled from the Markov chain, and $w$ are the emissions given 
the state sequence, so that the distribution of given $p_i = j$ is 
$w_i \sim e_{j}$.

For LDA, the parameters are the topics, $\phi = (\phi_1, \ldots, 
\phi_{\cS})$, with a prior $\phi_j \sim Dir_{\cW}(\eta)$, for some fixed 
$\eta >0$. The path is the set of topic assignments 
for each token in the corpus, and $w$ are the tokens, where $w_i$ is 
sampled from the topic $p_i$.

Next, the following definitions will be useful. We refer to the quantity 
\begin{equation}
\label{eq:path_likelihood}
\Prob{w,p | \phi}
\end{equation}
as the \textit{path likelihood}, and to 
\begin{equation}
\label{eq:data_likelihood}
\Prob{w | \phi} = \sum_{p\in \cP} \Prob{w,p | \phi}
\end{equation}
the \textit{data likelihood}.

Given the data $w$, we interested in the maximum likelihood parameter 
values, $\phi_{max}$, 
\begin{equation}
\label{eq:phi_argmax} 
 \phi_{max} = \argmax_{\phi}  \Prob{w | \phi}, 
\end{equation}
and in a sample of $p$ given $w$ and $\phi_{\max}$. 
A standard approach to obtaining an approximation of $\phi_{max}$ is 
to use a Gibbs sampler to sample from $\Prob{\phi, p | w}$.

Now, on the set of parameters and paths, $\Phi\times \cP$, define the function 
\begin{equation}
\nonumber
f(\phi,p) = \Prob{w,p, \phi} = \Prob{w,p | \phi} \cdot \Prob{\phi}.
\end{equation}
Denote by
\begin{equation}
\nonumber
C_f = \sum_{p} \int_{\Phi} f(\phi,p) d\phi = \Prob{w},	
\end{equation}
where $d \phi$ is the Lebesgue measure on $\Phi$ 
(that is, $\Prob{\phi}$ is a density with respect to $d\phi$ ). 

Then the probability density $\hat{f}(\phi,p) = \frac{1}{C_f} f(\phi,p)$ is, by definition, the density of the \textit{posterior} 
distribution $\Prob{\phi,p | w}$.

Observe that the posterior density of $\phi$ is proportional to the 
data likelihood of $\phi$:
\begin{eqnarray}
\label{eq:power_1_data_likelihood}
\hat{f}(\phi) = \sum_p \hat{f}(\phi,p) = \frac{1}{C_f} \sum_p \Prob{ \phi, p , \BD } = \\
= \frac{1}{C_f} \cdot \Prob{\phi} \cdot \Prob{w | \phi}. \nonumber
\end{eqnarray}

We are now ready to introduce the multiple paths model. The generative 
model consist of simply sampling the parameters 
$\phi$ from $\Prob{\phi}$, and then sampling $m$ independent copies 
of the paths, $p^j$, and the data, $w^j$. Here the data sequence 
$\Set{w_i^j}_{i \leq N}$ is sampled from the original generative model
given $\phi$ and $p^j$.  Therefore, the model generates $m$ paths, and $m
$ versions of the data. In what follows we will condition this model 
on the event that all the versions of the data coincide and are equal
to a given data sequence $w$. 

Specifically, given a data sequence $w$, we seek the posterior distribution of $\phi$ in the $m$ path model, given the data $w^j$, such that $w^j = w$ for 
all paths $j \leq m$. 

For $m\geq 1$, define on $\Phi \times \cP^{m}$ the function
\begin{equation}
\label{eq:fm_def}
f_m(\phi, p^1, \ldots, p^m) = \Prob{\phi} \cdot \prod_{j \leq m} \Prob{p^j, w | \phi}. 
\end{equation}
Denote by $C_{f_m}$ the corresponding normalization constant, 
\begin{eqnarray}
\nonumber
C_{f_m} =  \int_{\Phi} \sum_{p^1, \ldots, p^m} f_m(\phi, p^1, \ldots, p^m) d\phi. 
\end{eqnarray}

Then $\hat{f}_m = \frac{1}{C_{f_m}} f_m$ gives the distribution
of $(\phi, p^1, \ldots, p^m)$ with respect to the multipath generative 
model, conditioned on the event $w^j = w$ for all $j\leq m$. 

The following observation is the key fact of this section, and we record it as a lemma. 
\begin{lem}
Let $(\phi,p^1,\ldots,p^m)$ be a sample from $\hat{f}_m$. Then the marginal density of $\phi$ is
\begin{equation}
 \label{eq:power_m_data_likelihood}
 \Probu{m}{\phi | w}  = \frac{1}{C_{f_m}} \Prob{\phi} \Prob{w | \phi}^m.
\end{equation}
\end{lem}
\begin{proof}
Indeed, 
\begin{eqnarray*}
 C_{f_m} \cdot \Probu{m}{\phi | w}  = \sum_{p^1, \ldots, p^m} f_m(\phi, p^1, \ldots, p^m)  = \\
 \Prob{\phi} \sum_{p^1, \ldots, p^m} \prod_{j \leq m} \Prob{p^j, w | \phi} =  \\
 \Prob{\phi} \prod_{j \leq m} \sum_{p^j} \Prob{p^j, w | \phi} = \\
  \Prob{\phi} \Prob{w | \phi}^m.
\end{eqnarray*}
\end{proof}

Note that this generalizes (\ref{eq:power_1_data_likelihood}). The crucial detail about (\ref{eq:power_m_data_likelihood}) 
is that it contains the $m$-th power of the data likelihood. 
Let $\phi_{max}$, defined in (\ref{eq:phi_argmax}),  be a maximum likelihood solution given the data,
and let $\phi$ be any other parameter. Then, 
\begin{eqnarray*}
\Probu{m}{\phi_{max} | w} \big/ \Probu{m}{\phi | w} = \\
\Brack{\Prob{w | \phi_{max}} \big/ \Prob{w | \phi} }^m = \\
\Brack{ \Prob{\phi_{max} | w} \big/ \Prob{\phi | w} }^m. 
\end{eqnarray*}
In words, we are \textit{exponentially} more likely to obtain $\phi_{max}$ as a sample form $\hat{f}_m$ than from 
$\hat{f}$ (compared to any other $\phi$). Thus, as $m$ grows, any sample form $\hat{f}_m$ should yield $\phi$ with 
data likelihood close to that of $\phi_{max}$.

The discussion above shows why sampling from $\hat{f}_m$ yields better 
solutions than sampling from $\hat{f}$. We now discuss the Gibbs sampler 
for $\hat{f}_m$, which is a straightforward extension of the sampler for 
$\hat{f}$. The sampler is given schematically in Algorithm 
\ref{alg:m_path_gibbs}. For $m=1$, this is the standard Gibbs sampler. 

The notation $p_{-i}$ has 
the standard meaning of $p$ without the $i$-th 
component $p_i$. For the $m$-path model, our path is the collection of 
all paths $p^1,\dots,p^m$. However, when sampling 
and individual component $p_i^j$ from $p^j$, since the density 
(\ref{eq:fm_def}) factorizes over $j$, we obtain:

\begin{eqnarray*}
\Prob{p_i^j = s | p_{-i}^j, p^1,\ldots,p^{j-1},p^{j+1}, \ldots, p^m, w, \phi} = \\
\Prob{p_i^j = s |  p_{-i}^j , w, \phi} 
\end{eqnarray*}
for every $s \in \cS$. Thus, given the parameters $\phi$, one samples each path $p^j$ independently from the others. 

\begin{algorithm}
   \caption{Multi path Gibbs Sampler}
   \label{alg:m_path_gibbs}
\begin{algorithmic}[1]	
	\STATE {\bfseries Input:} Data $w$ of length $N$, number of paths $m$.
	\STATE {\bfseries Initialize} paths $p^1,\dots,p^m$ at random.
	\REPEAT
	\STATE Sample $\phi$ from $\Prob{\phi | p^1,\ldots,p^m, w}$.
	\FORALL{ $i \leq N$}
	\FORALL{ $j \leq m$}	
	\STATE Sample $p_i^j$ from $\Prob{p_i^j | p_{-i}^j, w, \phi}$. 
	\ENDFOR
	\ENDFOR
	\UNTIL{ iteration count threshold is reached.}    
\end{algorithmic}
\end{algorithm}

For the special cases of HMM and LDA, we write out Algorithm 
\ref{alg:m_path_gibbs} in full detail in Supplementary Material Section 
\ref{sec:explicit_implementations}. In particular, we detail the 
equations which describe how steps 4 and 7 of Algorithm 
\ref{alg:m_path_gibbs} are performed for these models. As mentioned 
earlier, these equations are straightforward modifications of the 
corresponding equations for the standard sampler. 
Also, note that for LDA, Algorithm 
\ref{alg:m_path_gibbs} corresponds to the so called partially collapsed 
Gibbs sampler, rather than to the more common fully collapsed version. 
However, a fully collapsed version of a sampler from $\hat{f_m}$ is 
equally easy to derive, and is also presented in Supplementary Material 
Section \ref{sec:explicit_implementations}.

As discussed above, the advantage of the multipath sampling 
comes from the fact that it implicitly samples $\phi$ from a distribution 
proportional to $\Prob{\phi | w}^m$. Note that while 
$\Prob{\phi | w}$ is not typically computable, the path likelihood 
$\Prob{\phi,p | w}$ \textit{is} computable. 
It is then natural to ask  what will happen if the Gibbs sampler is 
modified to sample the pair $(\phi,p)$ from a distribution proportional 
to $\Prob{\phi,p | w}^m$. This can easily be accomplished, using a single 
path. Such a sampler will spend much more time 
around the \textit{pairs} $(\phi,p)$ for which $\Prob{\phi,p | w}$ is 
high, compared to the regular Gibbs sampler. However, even if we were 
given a pair $(\phi',p')$ that maximizes 
$\Prob{\phi,p | w}$ globally over all $\phi,p$, it is straightforward to  
verify that $\phi'$ would not usually correspond to useful 
parameters. On the other hand, if the data is generated by a model with 
parameters $\psi$, then consistency results (such as, for instance, 
\cite{baum1966} for the case of HMMs) ensure that maximum data likelihood 
solutions, as in (\ref{eq:phi_argmax}), reconstruct $\psi$. The important 
feature of multipath sampler is therefore that it samples 
from $\Prob{\phi | w}^m$, even though $\Prob{\phi | w}$ is not 
computable.

Finally, it is worthwhile to remark that the multipath approach bears some 
resemblance to the 
EM algorithm. Indeed, in EM algorithm, for each $i\leq N$, one computes 
the \textit{distribution} of $p_i$ given all the data $w$ and parameters 
$\phi$. Then, given these distributions, one recomputes the parameters. 
One can view the multipath scheme as an approximation of the $p_i$ 
distribution with $m$ samples, instead of a computation of it in a closed 
form as in EM, or an approximation with one sample, as in regular Gibbs. 
A substantial difference between EM and the sampler is that our samples of 
$p_i$ are conditioned on the data \textit{and} the rest of the 
assignments, $p_{-i}$, while in EM the distribution is conditioned 
only on the data. However, interestingly, note that if the 
assignments $p_i$ are independent given the parameters, as in the case of 
 mixture models, than a multipath sampler directly 
approximates EM as $m$ grows.

\section{Experiments}
\label{sec:experiments}

In this Section we describe the HMM and LDA experiments. 
In addition, in Section \ref{sec:hmm_experiments}, after describing the 
HMM experiments we discuss Figure \ref{fig:c1_data_likelihoods}, referred 
to in Section \ref{sec:intro}, which demonstrates the local maxima behavior of 
the Gibbs sampler.

\subsection{HMMs}
\label{sec:hmm_experiments}

The HMM experiments use synthetic data. The data was generated by an HMM 
with two states, such that the transition
probability between the states is $0.45$, and the emissions are 
distributions on $W=10$ symbols. The 
emissions are show in Figure \ref{fig:emissions_rp5}  as the ground 
truth. All the experiments were conducted using the same 
fixed random sample of size $N=200000$ from this HMM.

These particular parameters were chosen in order to make the HMM hard to 
learn. Otherwise, all methods would succeed all the time, and would be 
difficult to compare. This corresponds to realistic situations where 
different emissions are hard to distinguish.

Note that while $N=200000$ may appear as a high number of samples, 
in this case, due to transitions probabilities which are close to $0.5$, 
this high number of samples is in fact necessary.  As we discuss later, 
even with this number of samples, both the EM algorithm and some runs 
of the Gibbs sampler manage to slightly overfit the solution.

To put the HMM learning problem in the Bayesian setting, we endow the set 
of parameters with a prior 
where transition and emission distributions are sampled from a uniform 
Dirchlet, $Dir(1)$. In other words, 
transition distributions are sampled uniformly from a 3-simplex, and 
emission distributions from an 
11-simplex.  The data $w$ is the given emissions of the HMM, and the path 
$p$ corresponds to the latent sequence of the states of the HMM.

Figure \ref{fig:c1_data_likelihoods} demonstrates the local maxima 
behavior. It shows the data log-likelihoods of a solution along a run of 
a collapsed Gibbs sampler, for 16 independent runs of a collapsed Gibbs 
sampler. 

Each run had 200000 iterations, with data log-likelihoods computed every 
100 iterations. As discussed in the Introduction, 
each run converges fairly quickly to an area of fixed likelihood, and 
stays there for extremely long time 
intervals.

\begin{figure}
  \centering
  \includegraphics[width=.5\textwidth]{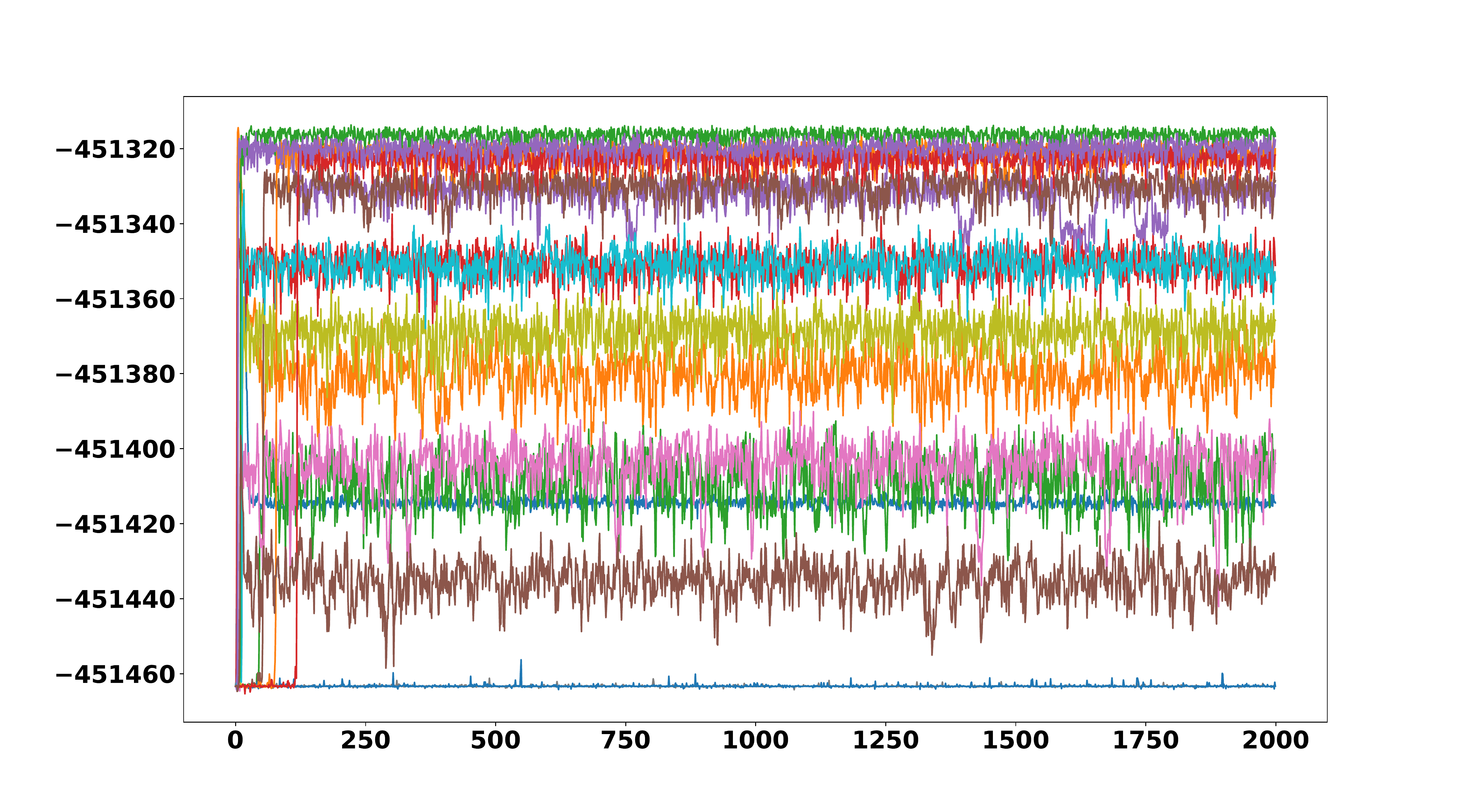}
  \caption{Collapsed Gibbs HMM Data Likelihoods, illustrating the Local Maxima. 
    Measurement is performed every 100 iterations. Total of 200000 iterations.}
  \label{fig:c1_data_likelihoods}
\end{figure}%

Figure \ref{fig:hmm_data_likelihood} shows the data log-likelihoods of 
the parameters $\phi$ for multiple runs of different algorithms after 200000 iterations. 
Each algorithm was run independently 16 times (all on the same data), and the resulting data 
log-likelihoods are plotted in ascending order. The algorithms included are the Collapsed Gibbs sampler for 1 and 
five paths (C 1, 5) and a non-collapsed sampler (PC 1, 5). 

For reference, the constant line (BW) corresponds to a solution found by the Baum-Welch algorithm, and Ground 
line corresponds to the data likelihood for the ground truth parameters. Note that the Ground line is slightly 
lower than the highest likelihood solutions, even with the relatively high number of samples $N=200000$. 

\begin{figure}
  \centering
  \includegraphics[width=.5\textwidth]{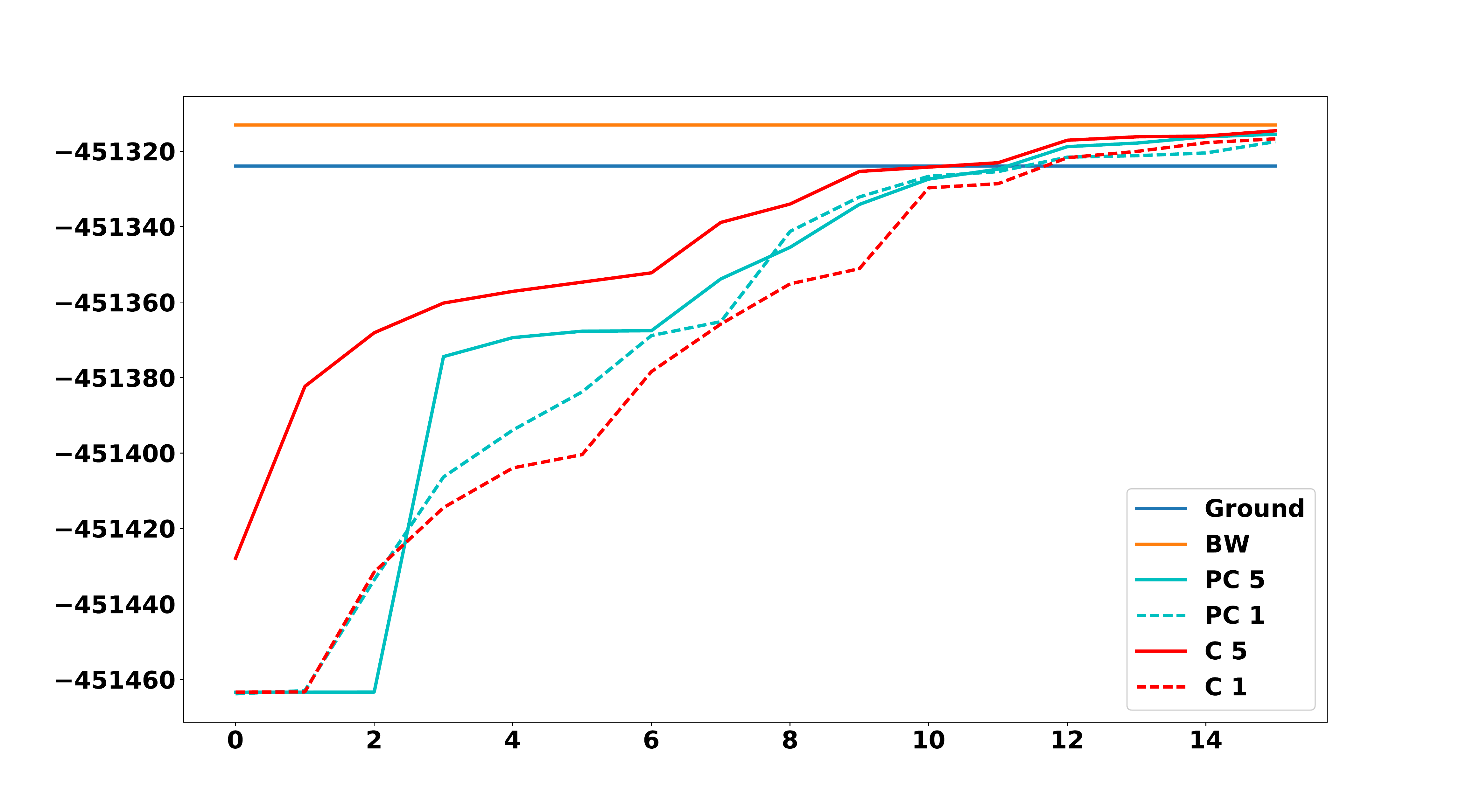}
  \caption{Data likelihoods, all algorithms}
  \label{fig:hmm_data_likelihood}
\end{figure}%

Clearly most of the algorithms improve on the base-line -- the Collapsed 
Gibbs sampler C1.  The the non-
collapsed sampler, PC1, which is also a base-line, provides results 
similar to those of C1. The 
best results are obtained by C 5. The emissions found by C5 run 
with the highest data likelihood are shown in Figure 
\ref{fig:emissions_rp5}. 
Note that the imperfect approximation of the emissions may be explained 
by a mild overfitting due to the barely sufficient 
amount of data.

\begin{figure}
  \centering
  \includegraphics[width=.5\textwidth]{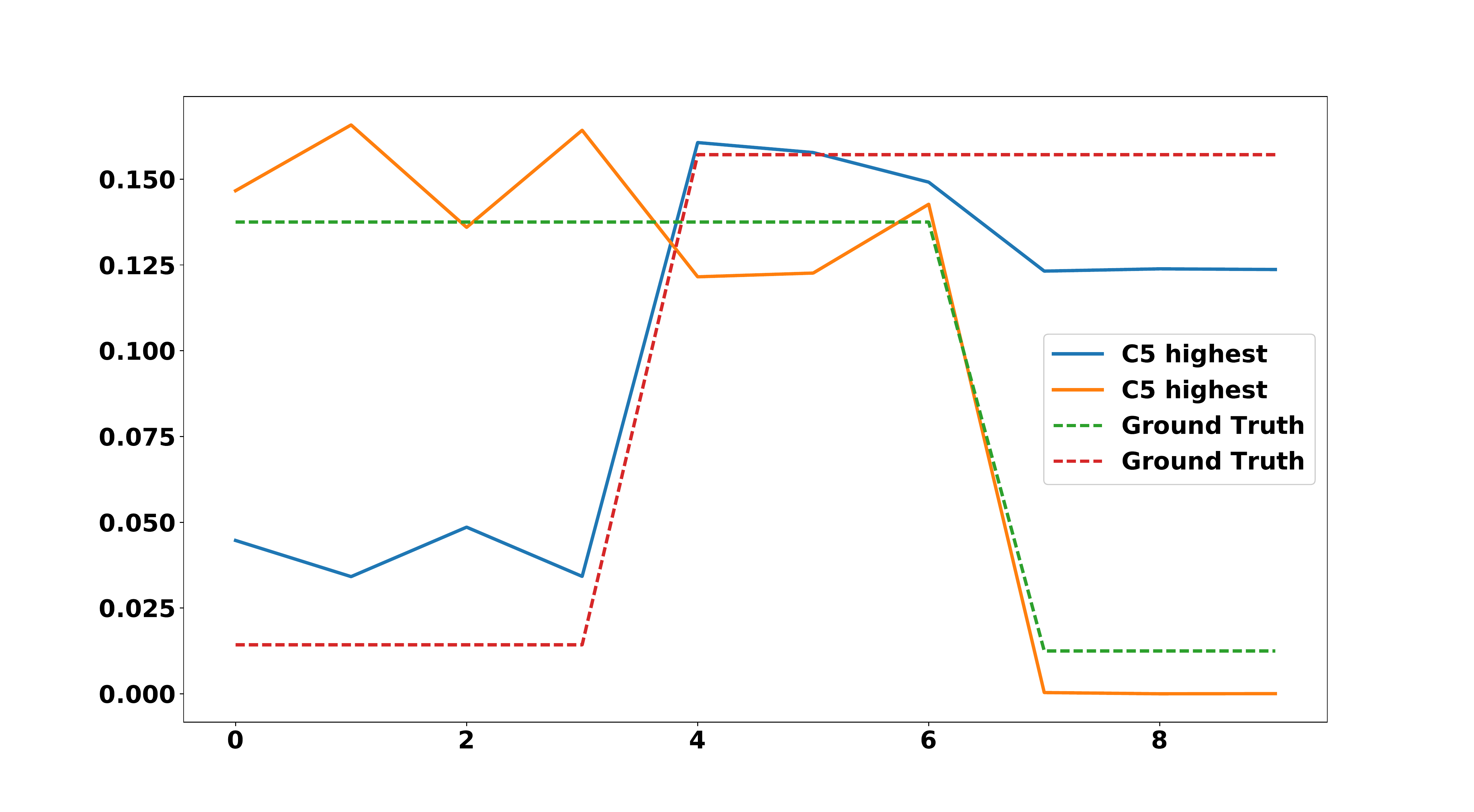}
  \caption{Emissions of highest C5}
  \label{fig:emissions_rp5}
\end{figure}%

\subsection{LDA}
\label{sec:lda_experiments}
In this Section we describe the LDA experiments. 

\subsubsection{Synthetic Data}
The $T = 10$ topics for the synthetic data are shown in Figure 
\ref{fig:lda_synth}. They were constructed as follows: Choose 10 overlapping bands of equal length (except the 
boundary ones) on a dictionary of size $W=100$, centered at equally spaced intervals over 
$0,\ldots,99$. Each topic is a linear combination of a band, with weight 
$0.95$ and a uniform distribution on $W$, with weight $0.05$.  Most of the 
points in the dictionary belong to two topics (if one disregards the 
uniform component present in all topics). 
\begin{figure}
  \centering
  \includegraphics[width=.5\textwidth]{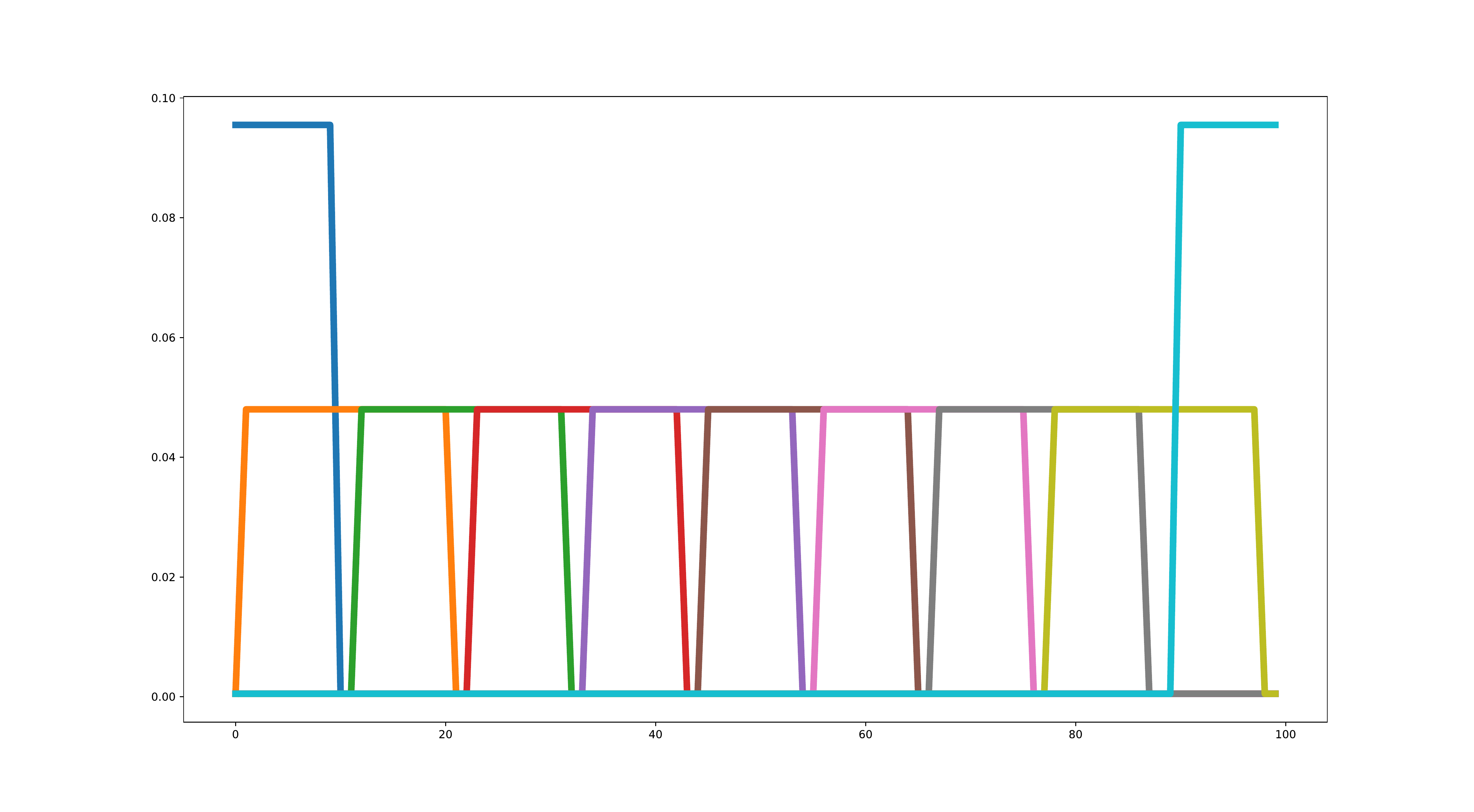}
  \caption{LDA Synthetic Topics}
  \label{fig:lda_synth}
\end{figure}%

Each document is generated by sampling the topic distribution 
$\theta \sim Dir(\alpha)$, with $\alpha = 1$, and sampling 
10 tokens by sampling a topic from $\theta$, and a token from that topic. 

We have generated a varying number of documents from this model, and 
have run the collapsed LDA Gibbs sampler with a varying number of paths. 
Given the topics $\beta_1, \ldots, \beta_{10}$ found by the sampler, 
to compute a distance to ground truth,$disc$, 
for each ground truth topic $
\beta^g_i$, we find a topic $\beta_j$ closest to 
$\beta^g_i$ in total variation distance, and take the mean of these 
distances:
\begin{equation}
disc = \frac{1}{10} \sum_{i} \min_j \norm{\beta^g_i - \beta_j}_{1}
\nonumber
\end{equation}
In addition, each $disc$ was averaged over 10 independent runs of the 
sampler (with same data and same number of paths). Samplers with different 
number of sources, but same number of documents, were run on the set of 
documents. All samplers were run for 10000 iterations, and in all cases
already after 3000 the topics stabilized and practically did not change 
further. The results are shown in Table \ref{tbl:synth_results}. 
Clearly, increasing the number of paths yields better reconstruction. 
Increasing the number of paths to more than 5 improved the performance 
further in some, but not all cases. 

\begin{table}[t]
\caption{Topic reconstruction results for synthetic LDA data, for varying number of paths and documents}
\label{tbl:synth_results}
\begin{center}
\begin{small}
\begin{sc}
\begin{tabular}{lcccr}
\toprule
 $m$ (paths num.) & 1500 & 3000 & 6000 & 9000 \\
\midrule
1    & 1.12 & 0.79 & 0.58 & 0.43\\
2    & 0.86 & 0.59 & 0.41 & 0.32\\
3    & 0.77 & 0.54 & 0.35 & 0.29\\
5    & 0.69 & 0.46 & 0.31 & 0.24\\
\bottomrule
\end{tabular}
\end{sc}
\end{small}
\end{center}
\vskip -0.1in
\end{table}

\subsubsection{SOTU Dataset}
\label{sec:sotu}
Our State of the Union speeches corpus contains speeches from years 1790 
to 2013. The 
preprocessing included stop words removal, stemming, removal of 40 most 
common words, and words that appeared only once in the corpus. 
Each speech was divided into paragraphs, and each paragraph
was treated as a document. After the preprocessing, the corpus contains 
around 767000 tokens, in about 21000 documents. There are on average 
50-80 documents corresponding to a single year. All the models were fit 
with $T=500$ topics and prior parameters $\eta = 0.01$ and $\alpha = 10/T 
= 0.02$. 

As discussed in Section \ref{sec:intro_experiments}, we compare the time-
concentration properties of LDA models on the State of the Union corpus, 
found by the collapsed Gibbs sampler (referred to as LDA in what follows), and by multipath collapsed Gibbs 
sampler with $m=5$ paths (to which we refer as mpLDA). We quantify the time concentration of topics in 
two different ways, via the computation of yearly entropy, and via the notion of quantile buckets, as discussed below.

Given the topic assignments $p$ returned by the sampler, we compute 
several quantities as discussed below. 
For each document $d$, let $\theta_d$ be the topic 
distribution ( a probability distribution on the set
$0,\ldots,T-1$) of the document, as computed by the sampler. 
Specifically, if $i_1^d, \ldots, i_{k}^d$ are the indexes of the tokens 
in a document $d$, with $k=\Abs{d}$ the size of the document, then 
\begin{equation*}
\theta_d = \frac{1}{k} \sum_{l \leq k} \delta_{p_{i_k^d}}. 
\end{equation*}

For the multipath sampler, we 
take the assignments from the first path $p^1$, for convenience. 
Typically, after convergence, for most tokens, 
all paths have an identical assignment for a given token.

Next, for each year, we compute the topic distribution: Let $\mathcal{Y}_y$ be the set of documents 
from year $y$. Then the topic distribution of the year $y$ is 
\begin{equation}
\nonumber
\theta_y = \frac{1}{\Abs{\mathcal{Y}_y}} \sum_{d \in \mathcal{Y}_y} \theta_d.
\end{equation}
We refer to $\theta_y(t)$ as the topic weight of topic $t$ for year $y$. 
The Figure \ref{fig:welfare_topic_compare} shows in blue the topic weights for all years of a particular topic 
$t$ in a model found by the collapsed Gibbs sampler. The 9 words with highest probabilities in this topic are indicated 
in the figure. In addition, we have found the topic $t'$ in the model returned by the Mulipath Gibbs, and $t''$ for the 
aternating sampler which are closest to $t$ in total variation distance (as distributions over words).  The topic weights of 
$t'$ and $t''$ are shown in orange and green. 

\begin{figure}
  \centering
  \includegraphics[width=.5\textwidth]{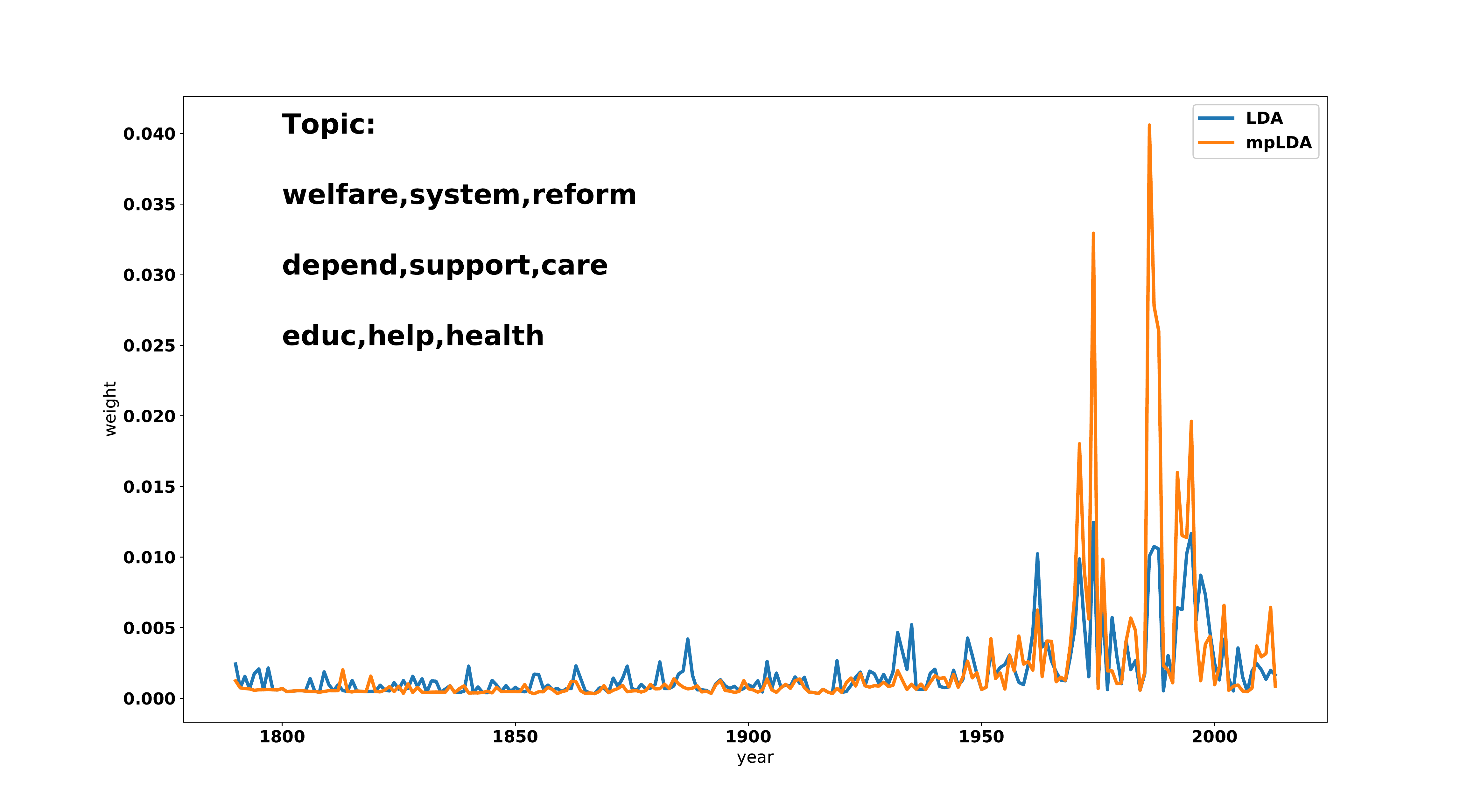}
  \caption{Topic Comparison LDA/mpLDA}
  \label{fig:welfare_topic_compare}
\end{figure}

As Figure \ref{fig:welfare_topic_compare} shows, the orange curve has more weight on the years where both
blue and orange curves have high values, and it has less weight in other areas. Note that the second statement here 
(less weight) does not follow from the first (more weight) -- the weight of the topic over the years does 
not have to sum to 1 (the sum over all topics for a fixed year is 1). We conclude that the orange topic 
is more time concentrated than the blue.  Similar conclusion also holds for the green topic.

We now proceed to show that the situation depicted in Figure \ref{fig:welfare_topic_compare} for a fixed topic, is in fact 
typical for most topics, and overall the topics of mpLDA are better concentrated than those of LDA. First, consider the 
entropy of the topic distribution $\theta_{y}$ for every year, for all models, as shown in Figure
\ref{fig:buckets_entropies} (right). The entropies of mpLDA are similar  consistently lower than those of LDA for every 
year. This means that each year in the new models is composed of slightly fewer topics, with higher weights.

Second, given the topic weights for a fixed topic, we quantify its time concentration via the notion 
of quantile bucket length. A $\gamma$-quantile bucket of the weights is a time interval which contains 
a $\gamma$ proportion of the total weight of the topic. For instance, Figure \ref{fig:quantile_buckets_demo}
show $\gamma=0.1$ quantile buckets for topics from Figure \ref{fig:welfare_topic_compare}. Shorter buckets
mean that a fixed proportion of the weight is contained in a shorter time interval. Note that the short
buckets of the orange curve in Figure \ref{fig:quantile_buckets_demo} are shorter than than those of the 
blue curve. 

\begin{figure}
  \centering
  \includegraphics[width=.5\textwidth]{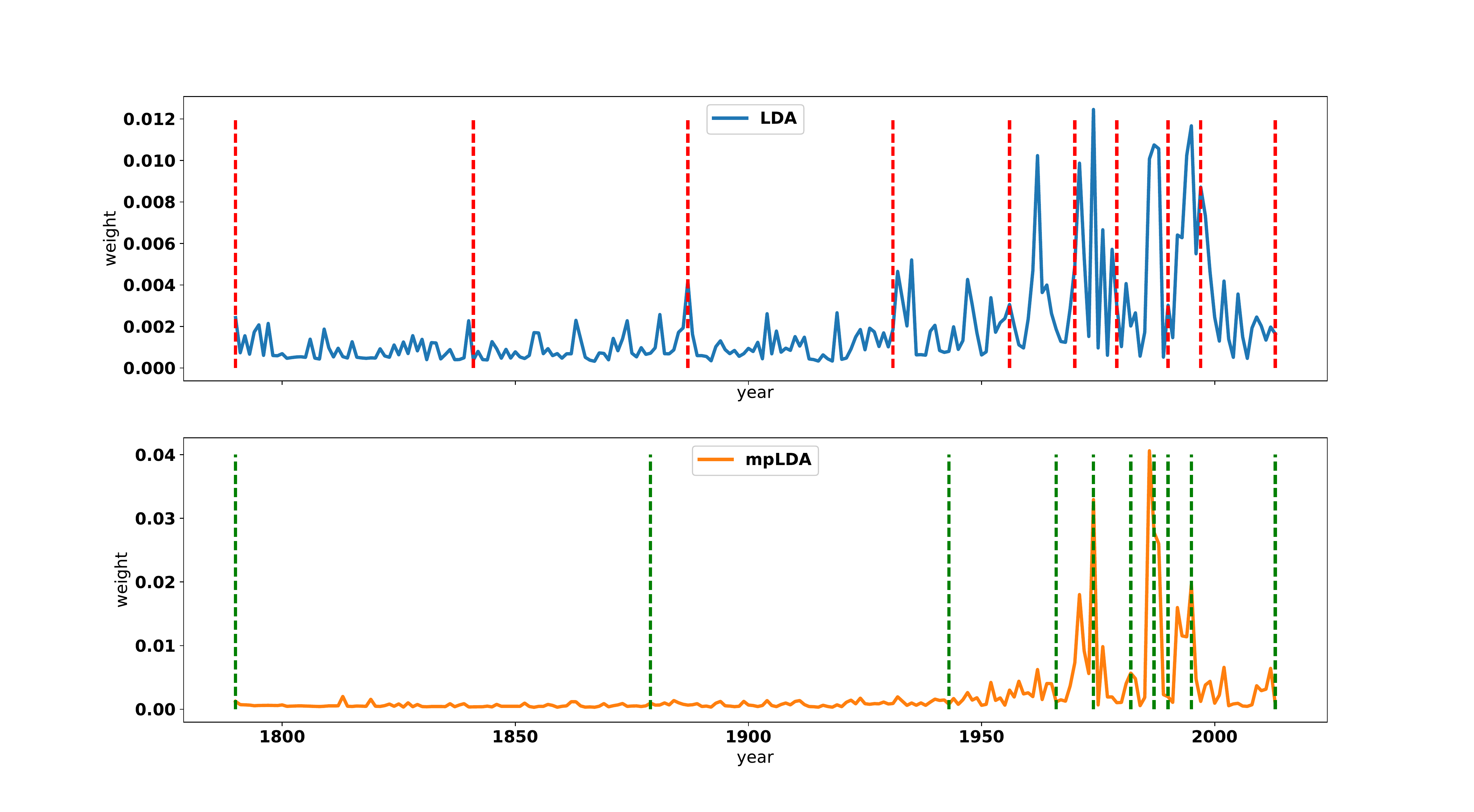}
  \caption{$0.1$-quantile buckets}
  \label{fig:quantile_buckets_demo}
\end{figure}

In Figure \ref{fig:buckets_entropies} (left), the distribution of the bucket lengths, with 
$\gamma=0.05$, for all the buckets from all the topics is shown, for all models. Since the sum 
of each 20 buckets contributed by a single topic is 224, the total number of years, the expectation
of each histogram must be $11.2$. We observe that the blue histogram is centered around this value, 
indicating more or less uniform (over time) weight distributions for most of the topics of LDA.
On the other hand, the orange histogram is clearly skewed to the left, which implies that there are
more short intervals containing a fixed weight in mpLDA. 

Some additional measurements related to this experiment may be found in 
Supplementary Material Section \ref{sec:sotu_time_supp}.
\begin{figure}
  \centering
  \includegraphics[width=.5\textwidth]{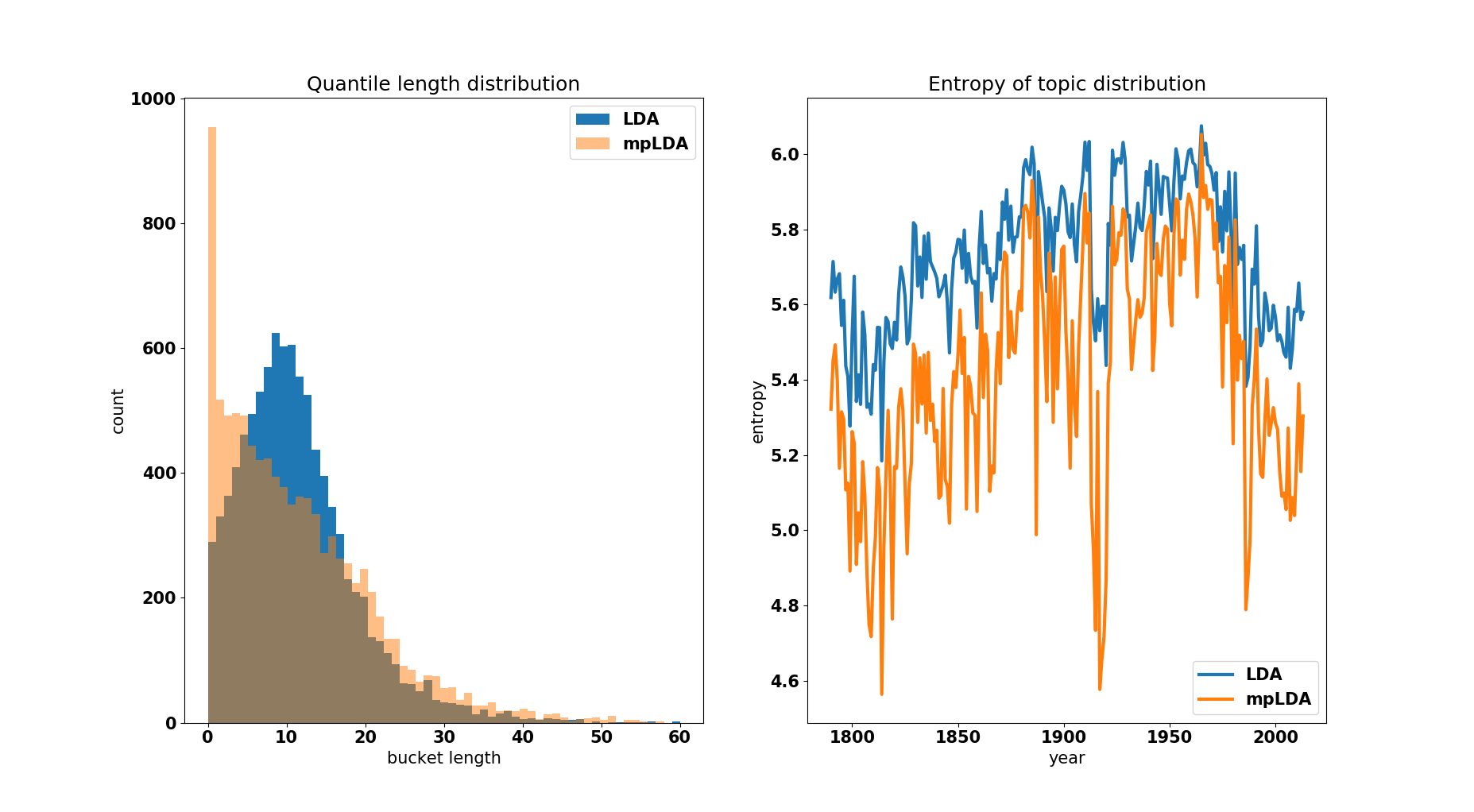}
  \caption{Bucket Lengths and Yearly Entropies}
  \label{fig:buckets_entropies}
\end{figure}

\section{Conclusion and Future Work}
\label{sec:conclusion}

In this work we have introduced a modification of the Gibbs sampler,
which combines samples from multiple sets of latent variables and yet 
avoids the identifiability problem. We have shown that asymptotically 
this sampler is exponentially more likely, compared to the regular Gibbs 
sampler, to return samples which have high data likelihood in the 
original generative model. We have also shown empirically that the quality 
of the solutions indeed improves, as the number of paths grows. 

In addition, we found that time concentration properties of topic
models may improve when using the multipath sampler, compared to 
regular Gibbs. This is one instance of a phenomenon where a better 
optimization of a model yields better insight into the data. 
However, full implications of this fact, in particular in the context
of trend analysis, would require a more extensive investigation.  

While we have tested our methods on the HMM and LDA models, our methods 
were formulated in a fairly generic latent variables settings, and could 
be extended even further. Thus, a natural direction for future work would 
be to understand the impact of the methods presented here on
learning in other models, such as Markov Random Fields or extended HMMs 
with large state spaces, for which EM algorithms are not feasible. 

\bibliographystyle{apalike}
\bibliography{multipath_bib}

\newpage

\appendix

\section{Additional Considerations Regarding Time Concentration}
\label{sec:sotu_time_supp}
In Section \ref{sec:sotu} we have found that for topics and topic 
assignments from the multipath Gibbs sampler, the distribution of bucket 
lengths is more skewed  to the left. 

However, note that there might have existed degenerate
topic assignments, which could have produced the concentration results of Figure \ref{fig:buckets_entropies}. 
Indeed, consider an assignment where there is a single topic that is assigned to a major portion of every document, 
with the rest of the topics distributed in a concentrated way over the years. Such an assignment would lower the yearly 
entropy and produce a left skew in bucket lengths, but, arguably, should not be considered a better description of the data. 
To ensure that this is not the case, for each topic $t \leq T$, consider the total topic weight
\begin{equation}
w_t = \sum_y \theta_y(t). 
\end{equation}
In Figure \ref{fig:buckets_weighted}, we show the sorted topic weights for both models (right) and the histogram of bucket 
lengths (left), where each bucket length is weighted proportionally to the topic weight $w_t$ of its topic (rather then by 1, 
as in computation for Figure \ref{fig:buckets_entropies}). The topic weight distribution is practically the same for both models, and the 
lengths histograms also remain practically unchanged. This shows that the mpLDA genuinely redistributes the same topic 
weights in a more time concentrated fashion, rather than produces anomalous assignments. 

\begin{figure}
  \centering
  \includegraphics[width=.5\textwidth]{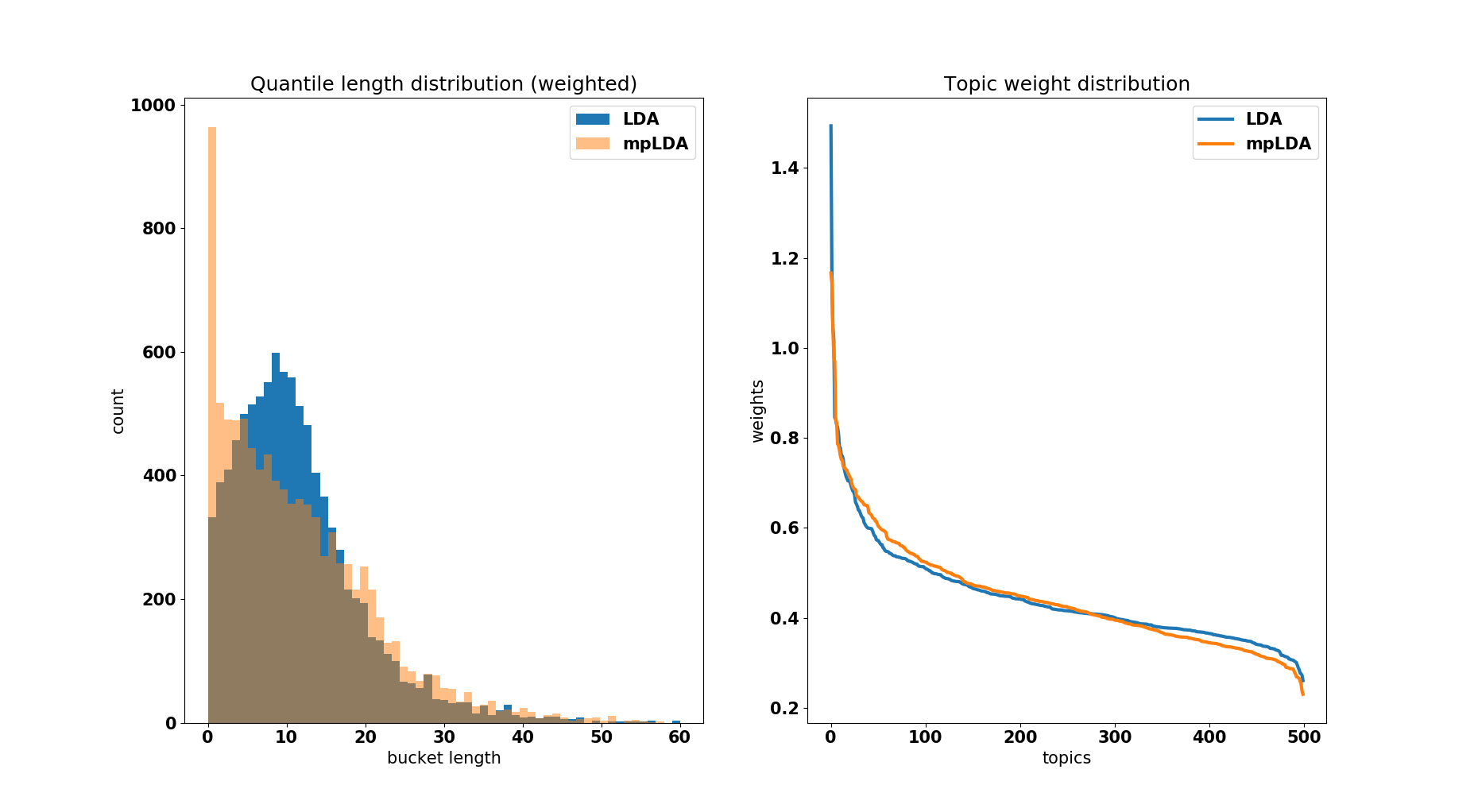}
  \caption{Weighted Bucket Lengths, Topic Weights}
  \label{fig:buckets_weighted}
\end{figure}

\section{Explicit Implementations}
\label{sec:explicit_implementations}

In this Section we describe a fully detailed implementation of a 
multipath Gibbs sampler Algorithm \ref{alg:m_path_gibbs} for LDA. 
As discussed in Section \ref{sec:multiple_paths}, in order to fully 
specify Algorithm \ref{alg:m_path_gibbs} for a particular generative 
model, one has to implement the sampling in lines 4 and 7 of the 
Algorithm. Line 4 corresponds to line 6 in Algorithm \ref{alg:lda_gibbs}, 
and line 7 to lines 12-19.

In addition, we describe the multipath analog of the \textit{collapsed} Gibbs 
sampler. Schematically, the multipath collapsed Gibbs sampler is given in 
Algorithm \ref{alg:m_path_collapsed_gibbs_schem}, and 
Algorithm \ref{alg:lda_collapsed_gibbs} is a full implementation. 

The notation for the algorithms is summarized in Table 
\ref{tbl:lda_notation}, and is analogous to the one commonly 
used in the literature. 

Note that for $m=1$ both Algorithm \ref{alg:lda_gibbs} and
\ref{alg:lda_collapsed_gibbs} coincide with their standard counterparts.

All computations in the derivation of the algorithms are standard 
Dirichlet conjugacy type 
computations. We omit the computations, since they are completely 
analogous to the computations for the 
derivation of the regular LDA Gibbs sampler. The algorithms for the
HMM model are also derived using similar (slightly simpler) computations. 

Please see the next page for the notation and the algorithms. 

\begin{algorithm}
   \caption{Multipath Collapsed Gibbs Sampler}
   \label{alg:m_path_collapsed_gibbs_schem}
\begin{algorithmic}[1]	
	\STATE {\bfseries Input:} Data $w$ of length $N$, number of paths $m$.
	\STATE {\bfseries Initialize} paths $p^1,\dots,p^m$ at random.
	\REPEAT
	\FORALL{ $i \leq N$}
	\FORALL{ $j \leq m$}	
	\STATE Sample $p_i^j$ from \\ 
	       \spaceo $\Prob{p_i^j | p^1,\ldots,p^j_{-i},\ldots,p^m, w}$.
	\ENDFOR
	\ENDFOR
	\UNTIL{ iteration count threshold is reached.}    
\end{algorithmic}
\end{algorithm}

\begin{table*}[t]
  \caption{LDA Notation}
  \label{tbl:lda_notation}
  \centering
  \begin{tabular}{ll}
    \toprule
    Name     & Description      \\
    \midrule
	\multicolumn{2}{c}{Model Notation}  \\
	\midrule	
    $T$  & Number of topics   \\
    $W$     & Dictionary size  \\
    $m$     & Number of paths in a multipath model       \\    
    $\eta$   & Hyperparameter for words in a topic       \\
    $\alpha$   & Hyperparameter for topics in a document \\
    $\beta_i$       & Topics, $i\leq T$. \\
    \midrule
	\multicolumn{2}{c}{Corpus Notation} \\
	\midrule	
	$N$      & Total number of tokens in a corpus \\
	$D$      & Total number of documents in a corpus \\
	$w_i$ &  The $i$-th token in a corpus. $i\leq N$ and 
							$w_i \leq W \spaceo \forall i \leq N$   \\
	$\mathcal{D}_i$ & The document of the $i$-th token in a corpus. $i\leq N$ and 
							$\mathcal{D}_i \leq D \spaceo \forall i \leq N$   \\
    \midrule
	\multicolumn{2}{c}{Sampler Notation}   \\
	\midrule	
	$p^j_i$ & Topic assignment to the $i$-th word in a corpus. $i\leq N$,$j \leq m$ and 
							$p^j_i \leq T \spaceo \forall i,j $    \\
	$\ctw_{tw}$  & Topic-word counter, $\ctw_{tw} = \sum_{j \leq m} \Abs{\Set{i \leq N \setsep 
                                                       w_i = w \wedge 
                                                       p^j_i = t
                                                       }
                                                       }
                                                       $   \\

	$\ct_t$ & Topic occurrence count, $\ct_t = \sum_{w} \ctw_{tw}$ \\				
	$\cd_d $ & Document occurrence count (size of document $d$), $\cd_d = \Abs{\Set{i\leq N \setsep \mathcal{D}_{i} = d}}$  \\	
    $\cdt_{jdt}$ & Document-topic counter, $\cdt_{jdt} = 
													\Abs{\Set{i \leq N \setsep 
                                                       \mathcal{D}_i = d \wedge 
                                                       p^j_i = t
                                                       }
                                                       }
                                                       $ with $j \leq m$  \\
    \bottomrule
  \end{tabular}
\end{table*}

\begin{algorithm*}[t]
   \caption{Multipath Partially Collapsed LDA Gibbs Sampler}
   \label{alg:lda_gibbs}
\begin{algorithmic}[1]	
	\STATE {\bfseries Input:} Hyperparameters $\alpha,\eta$, data $\mathcal{W},\mathcal{D}$.
	\STATE {\bfseries Initialize} $p^j_i$ to random integers in range $[1,T]$.
	\STATE {\bfseries Initialize} $\ctw, \cdt, \ct$ from $p$.
	\REPEAT

	\FORALL{ $t \leq T$}
	\STATE Sample $\beta_t \gets Dir(\eta + \ctw_{t,0}, \ldots, \eta + \ctw_{t,W-1})$
	\ENDFOR
	
	\FORALL{ $j \leq m$}
	\FORALL{ $i \leq N$}

	\STATE $w \gets w_i$, $d \gets \mathcal{D}_i$. 
	\STATE $z \gets p^j_i$.
	\STATE $\ctw_{zw} \gets \ctw_{zw} -1$,\spaceo $\cdt_{jdz} \gets \cdt_{jdz} -1$, \spaceo
					$\ct_{z} \gets \ct_{z} -1$.
	\FORALL{ $t \leq T$}
	\STATE $r_t \gets \beta_t(w) 
	   \cdot \Brack{\cdt_{jdt} + \alpha}$.
	\ENDFOR
	\STATE $r \gets \sum_{t\leq T} r_t$.	
	\STATE Sample $z' \gets multinomial\Brack{\frac{r_1}{r}, \ldots,\frac{r_T}{r} }$.
	\STATE $p^j_i \gets z'$.
	\STATE $\ctw_{z' w} \gets \ctw_{z'w} +1$,\spaceo $\cdt_{jdz'} \gets \cdt_{jdz'} +1$, \spaceo
					$\ct_{z'} \gets \ct_{z'} +1$.
	\ENDFOR
	\ENDFOR

	\UNTIL{ iteration count threshold is reached.}    
\end{algorithmic}
\end{algorithm*}

\begin{algorithm*}[t]
   \caption{Multipath Collapsed LDA Gibbs Sampler}
   \label{alg:lda_collapsed_gibbs}
\begin{algorithmic}[1]	
	\STATE {\bfseries Input:} Hyperparameters $\alpha,\eta$, data $\mathcal{W},\mathcal{D}$.
	\STATE {\bfseries Initialize} $p^j_i$ to random integers in range $[1,T]$.
	\STATE {\bfseries Initialize} $\ctw, \cdt, \ct$ from $p$.
	\REPEAT

	\FORALL{ $j \leq m$}

	\FORALL{ $i \leq N$}
	\STATE $w \gets w_i$, $d \gets \mathcal{D}_i$. 
	\STATE $z \gets p^j_i$.
	\STATE $\ctw_{zw} \gets \ctw_{zw} -1$,\spaceo $\cdt_{jdz} \gets \cdt_{jdz} -1$, \spaceo
					$\ct_{z} \gets \ct_{z} -1$.
	\FORALL{ $t \leq T$}
	\STATE $r_t \gets \frac{\ctw_{tw} + \eta}{\ct_{t} + T \cdot \eta} 
	   \cdot \Brack{\cdt_{jdt} + \alpha}$.
	\ENDFOR
	\STATE $r \gets \sum_{t\leq T} r_t$.	
	\STATE Sample $z' \gets multinomial\Brack{\frac{r_1}{r}, \ldots,\frac{r_T}{r} }$.
	\STATE $p^j_i \gets z'$.
	\STATE $\ctw_{z' w} \gets \ctw_{z'w} +1$,\spaceo $\cdt_{jdz'} \gets \cdt_{jdz'} +1$, \spaceo
					$\ct_{z'} \gets \ct_{z'} +1$.
	\ENDFOR
	\ENDFOR

	\UNTIL{ iteration count threshold is reached.}    
\end{algorithmic}
\end{algorithm*}

\end{document}